\documentclass[final]{cvpr}
\usepackage{times,epsfig,graphicx,amsmath,amssymb,subcaption}

\allowdisplaybreaks
\usepackage[ruled,vlined,linesnumbered]{algorithm2e}
\usepackage{float,amsthm,enumitem,array}
\usepackage[table,x11names]{xcolor}

\newcommand{\la}{\left\langle}
\newcommand{\ra}{\right\rangle}
\newcommand{\norm}[1]{\left\lVert#1\right\rVert} 
\newcommand{\LL}{\mathcal{L}}
\newcommand{\E}{\mathbb{E}}
\newcommand{\PP}{{\mathbb P}}
\newcommand{\data}{{\rm d}}
\newcommand{\model}{{\rm g}}

\newtheorem{thm}{Theorem}

\usepackage[pagebackref=true,breaklinks=true,colorlinks,bookmarks=false]{hyperref}

\def\cvprPaperID{7269} 
\def\confYear{CVPR 2021}

\begin{document}

\title{Adaptive Weighted Discriminator for Training Generative Adversarial Networks}

\author{Vasily Zadorozhnyy$^{1}$, Qiang Cheng$^{2,}$\thanks{Research supported in part by NSF OIA 2040665, NIH UH3 NS100606-05, and R01 HD101508-01 grants.}\:\,, Qiang Ye$^{1,}$\thanks{Research supported in part by NSF under grants DMS-1821144 and DMS-1620082.}\\
$^{1}$ Department of Mathematics\\
$^{2}$ Institute for Biomedical Informatics, Departments of Computer Science and Internal Medicine\\
University of Kentucky, Lexington, Kentucky 40506-0027\\
{\tt\small \{vasily.zadorozhnyy, Qiang.Cheng, qye3\}@uky.edu}
}

\maketitle
\thispagestyle{empty} 

\begin{abstract}
    Generative adversarial network (GAN) has become one of the most important neural network models for classical unsupervised machine learning. A variety of discriminator loss functions have been developed to train GAN's discriminators and they all have a common structure: a sum of real and fake losses that only depends on the actual and generated data respectively. One challenge associated with an equally weighted sum of two losses is that the training may benefit one loss but harm the other, which we show causes instability and mode collapse. In this paper, we introduce a new family of discriminator loss functions that adopts a weighted sum of real and fake parts, which we call adaptive weighted loss functions or aw-loss functions. Using the gradients of the real and fake parts of the loss, we can adaptively choose weights to train a discriminator in the direction that benefits the GAN's stability. Our method can be potentially applied to any discriminator model with a loss that is a sum of the real and fake parts. Experiments validated the effectiveness of our loss functions on unconditional and conditional image generation tasks, improving the baseline results by a significant margin on CIFAR-10, STL-10, and CIFAR-100 datasets in Inception Scores (IS) and Fr\'echet Inception Distance (FID) metrics.
\end{abstract}

\section{Introduction}\label{sect:intro}

Generative Adversarial Network (GAN)~\cite{NIPS2014_5423} has become one of the most important neural network models for unsupervised machine learning. The origin of this idea lies in the combination of two neural networks, one generative and one discriminative, that work simultaneously. The task of the generator is to generate data of a given distribution, while the discriminator's purpose is to try to recognize which data are created by the generative model and which are the original ones. While a variety of GAN models have been developed, many of them are prone to issues with training such as instability where model parameters might destabilize and not converge, mode collapse where the generative model produces a limited number of different samples, diminishing gradients where the generator gradient vanishes and training does not occur, and high sensitivity to hyperparameters.

In this paper, we focus on the discriminative model to rectify the issues of instability and mode collapse in training GAN. In the GAN architecture, the discriminator model takes samples from the original dataset and the output from the generator as input and tries to classify whether a particular element in those samples is {\em real} or {\em fake data}~\cite{NIPS2014_5423}. The discriminator updates its parameters by maximizing a discriminator loss function via backpropagation through the discriminator network. In many of the proposed models~\cite{NIPS2014_5423, WGAN-GP, ls17gan, lim2017geometric}, the discriminator loss function consists of two equally weighted parts: the ``real part'' that purely relies on the original dataset and the ``fake part'' that depends on the generator network and its output; for simplicity we will call them $\LL_r$ and $\LL_f$ for \emph{real} and \emph{fake} losses, respectively. For example, in the original GAN paper~\cite{NIPS2014_5423}, the discriminator loss function $\LL_D$ is written as

\begin{equation}
    \LL_D=\LL_r+\LL_f,\label{intro:eq:01}
\end{equation}
with $\LL_r=\E_{x\sim p_d}\big[\log D(x) \big]$ and $\LL_f=\E_{z\sim p_z}\big[\log (1-D(G(z))) \big]$, where $D$ and $G$ are the discriminative and generative models, respectively, $p_d$ is the probability distribution of the real data, and $p_z$ is the probability distribution of the generator parameter $z$. 

The goal of the GAN discriminator training is to increase both $\LL_r$ and $\LL_f$ so that the discriminator $D(\cdot)$ assigns high scores to real data and low scores to fake data. This is done in (\ref{intro:eq:01}) by placing equal weights on $\LL_r$ and $\LL_f$~\cite{NIPS2014_5423}. However, the training with $\LL_D$ is not performed equally on $\LL_r$ and $\LL_f$. Indeed, a gradient ascent training step along the $\nabla \LL_D$ may decrease $\LL_r$ (or $\LL_f$), depending on the angle between $\nabla \LL_D$ and $\nabla \LL_r$ (or $\nabla \LL_f$). For example, if we have a large obtuse angle between $\nabla \LL_r$ and $\nabla \LL_f$, which is the case in most training steps (see \S \ref{sec:s1}), training along the direction of $\nabla \LL_D$ may potentially decrease either $\LL_r$ or $\LL_f$ by going in the opposite direction to $\nabla \LL_r$ or $\nabla \LL_f$ (see \S \ref{sect:awGAN_section} and \S \ref{sec:s2}). We suggest that this reduction on the real loss may destabilize training and cause mode collapses. Specifically, if a generator is converging with its generated samples close to the data distribution (or a particular mode), a training step that increases the fake loss will reduce the discriminator scores on the fake data and, by the continuity of $D(\cdot)$, reduce the scores on the nearby real data as well. With the updated discriminator now assigning lower scores to the regions of data where the generator previously approximated well, the generator update is likely to move away from that region and to the regions with higher discriminator scores (possibly a different mode). Hence, we see instability or mode collapse. See \S \ref{sec:s3} for experimental results.

We propose a new approach in training the discriminative model by modifying the discriminator loss function and introducing adaptive weights in the following way,
\begin{equation}
    \LL_D^{aw}=w_r\cdot \LL_r + w_f\cdot \LL_f\label{intro:eq:02}.
\end{equation}

We adaptively choose $w_r$ and $w_f$ weights to calibrate the training in the real and fake losses. Using the information of $\nabla \LL_r$ and $\nabla \LL_f$, we can control the gradient direction, $\nabla \LL_D^{aw}$, by either training in the direction that benefits both $\LL_r$ and $\LL_f$ or increasing one loss while not changing the other. This attempts to avoid a situation where training may benefit one loss but significantly harm the other. A more detailed mathematical approach is presented in \S \ref{sect:awGAN_section}.

Our proposed method can be applied to any GAN model with a discriminator loss function composed of two parts as in (\ref{intro:eq:01}). For our experiments we have applied adaptive weights to the SN-GAN~\cite{miyato2018spectral}, AutoGAN~\cite{Gong_2019_ICCV}, and BigGAN~\cite{brock2018large} models for unconditional as well as conditional image generating tasks. We have achieved significant improvements on them for CIFAR-10, STL-10 and CIFAR-100 datasets in both Inception Scores (IS) and Fr\'echet Inception Distance (FID) metrics, see \S \ref{sect:experiments}. Our code is available at \href{https://github.com/vasily789/adaptive-weighted-gans}{https://github.com/vasily789/adaptive-weighted-gans}.

{\bf Notation: } We use $\la\cdot,\cdot\ra_2$ to denote the Euclidean inner product, $\norm{x}_2$ the Euclidean 2-norm, and $\angle_2(x,y):=\arccos\left(\frac{\la x,y \ra_2}{\norm{x}_2\norm{y}_2}\right)$ the angle between vectors $x$ and $y$.

\section{Related Work}\label{sect:related_work}
GAN was first proposed in~\cite{NIPS2014_5423} for creating generative models via simultaneous optimization of a discriminative and a generative model. The original GAN may suffer from vanishing gradients during training, non-convergence of the model(s), and mode collapse; see~\cite{che2016mode,NIPS2017_6779,Unrolled_GAN,radford2015unsupervised,IS_NIPS2016_6125} for discussions. Several papers~\cite{wgan_orig, WGAN-GP, lim2017geometric,ls17gan} have addressed the issues of vanishing gradients by introducing new loss functions. The LSGAN proposed in~\cite{ls17gan} adopted the least squares loss function for the discriminator that relies on minimizing the Pearson $\chi^2$ divergence, in contrast to the Jensen--Shannon divergence used in GAN. The WGAN model~\cite{wgan_orig, WGAN-GP} introduced another way to solve the problem of convergence and mode collapse by incorporating Wasserstein-1 distance into the loss function. As a result, WGAN has a loss function associated with image quality which improves learning stability and convergence. The hinge loss function introduced in~\cite{lim2017geometric, tran2017hierarchical} achieved smaller error rates than cross-entropy, being stable against different regularization techniques, and having a low computational cost~\cite{dong2019deeper}. The models in~\cite{binkowski2018demystifying,NIPS2017_6815,DBLP:journals/corr/abs-1708-08819} adopted a loss function called maximum mean discrepancy (MMD). A repulsive function to stabilize the MMD-GAN training was employed in~\cite{wang2018improving}, and the MMD loss function was weighted in~\cite{diesendruck2019importance} according to the contribution of data to the loss function. \cite{DBLP:journals/corr/abs-1709-03831} presented a dual discriminator GAN that combines two discriminators in a weighted sum. 

New loss functions are not the only way of improving GAN's framework. DCGAN~\cite{radford2015unsupervised}, one of the first and more significant improvements in the GAN architecture, was the incorporation of deep convolutional networks. The Progressive Growing GAN~\cite{karras2018progressive} was created based on~\cite{wgan_orig, WGAN-GP} with the main idea of progressively adding new layers of higher resolution during training, which helps to create highly realistic images. \cite{Gong_2019_ICCV, doveh2019degas, tian2020offpolicy} developed neural architecture search methods to find an optimal neural network architecture to train GAN for a particular task. 

There are many works dedicated to the conditional GAN, for example BigGAN~\cite{brock2018large} which utilized a model with a large number of parameters and larger batch sizes showing a significant benefit of scaling.

There are many works devoted to improving or analyzing GAN training. \cite{Unrolled_GAN} trained the generator by optimizing a loss function unrolled from several training iterations of the discriminator training. SN-GAN~\cite{miyato2018spectral}, normalized the spectral norm of each weight to stabilize the training. Recent work~\cite{sanyal2019stable} introduced stable rank normalization that simultaneously controls the Lipschitz constant and the stable rank of a layer. \cite{pmlr-v80-li18d} developed an analysis to suggest that  first-order approximations of the discriminator lead to instability and mode collapse. \cite{nagarajan2017gradient} proved local stability under the model that both the generator and the discriminator are updated simultaneously via gradient descent. \cite{Chu2020Smoothness} analyzed the stability of GANs through stationarity of the generator. \cite{Mescheder2018ICML} points out that absolute continuity is necessary for GAN training to converge. Relativistic GAN~\cite{jolicoeur-martineau2018} addressed the observation that with generator training increasing the probability that fake data is real, the probability of real data being real would decrease. \cite{Binkowski2019BatchWF} proposed a method of re-weighting training samples to correct for mass shift between the transferred distributions in the domain transfer setup. \cite{che2020gan} viewed GAN as an energy-based model and proposed an MCMC sampling based  method. 

\section{Adaptive Weighted Discriminator} \label{sect:awGAN_section}
In GAN training, if we maximize $\LL_D$ to convergence in training discriminator $D$, we should meet the goal to increase  both $\LL_r$ and $\LL_f$. However, in practice, 
we train with a gradient ascent step along $\nabla \LL_D =\nabla \LL_r+\nabla \LL_f$, which may be dominated by either $\nabla \LL_r$ or $\nabla \LL_f$. Then, the training may be done primarily on one of the losses, either $\LL_r$ or $\LL_f$. 
Consider a gradient ascent training iteration for $\LL_D$, 
\begin{equation}
    \theta_1 \longleftarrow \theta_0 + \lambda \nabla \LL_D,
    \label{eq:ga}
\end{equation}
where $\lambda$ is a learning rate. Then using the Taylor Theorem, we can expand both $\LL_r$ and $\LL_f$ about $\theta_0$, 

\begin{align}
    \LL_r(\theta_1)&=\LL_r(\theta_0) + \lambda \nabla \LL_r^T \nabla\LL_D+\mathcal{O}(\lambda^2)\\
    &= \LL_r(\theta_0)\nonumber\\
    &\quad +\lambda \norm{\nabla\LL_r}_2\norm{\nabla \LL_D}_2 \cos\left(\angle_2 \left(\nabla \LL_r, \nabla \LL_D\right)\right) \nonumber\\
    & \quad\quad + \mathcal{O}(\lambda^2) \label{awGAN:eq:01}
\end{align}
and
\begin{align}
    \LL_f(\theta_1)&= \LL_f(\theta_0)\nonumber\\
    &\quad +\lambda \norm{\nabla\LL_f}_2\norm{\nabla \LL_D}_2 \cos\left(\angle_2 \left(\nabla \LL_f, \nabla \LL_D\right)\right)\nonumber\\
    & \quad\quad + \mathcal{O}(\lambda^2), \label{awGAN:eq:02}
\end{align}
where we have omitted the evaluation point $\theta_0$ in all gradients (i.e. $\nabla \LL_*=\nabla\LL_*(\theta_0)$) to avoid cumbersome expressions. If one of $\angle_2 \left(\nabla \LL_r, \nabla \LL_D\right)$ and $\angle_2 \left(\nabla \LL_f, \nabla \LL_D\right)$ is obtuse, then to the first order approximation, the corresponding loss is decreased. This causes a decrease in the discriminator assigning a correct  score $D(\cdot)$ to the real (or fake) data. Thus, a gradient ascent step with loss (\ref{intro:eq:01}) may turn out to decrease one of the losses if the angle $\angle_2 \left(\nabla \LL_r, \nabla \LL_f\right) > 90^{\circ}$. This situation occurs quite often in GAN training; see \S \ref{sec:s1} for some experimental results illustrating this. 

This undesirable situation is expected to happen in GAN training when the generator has produced samples close to the data distribution or its certain modes. If a training step in the direction $\nabla \LL_D$ results in an increase in the fake loss or equivalently a decrease in  the discriminator scores $D(G(z))$ on the fake data, it will decrease the scores  $D(x)$ on the real data as well by the continuity of $D(\cdot)$. Equivalently, this reduces the real loss. With the updated discriminator assigning lower scores to the regions of the data where the generator previously approximated well, the generator update using the new discriminator will likely move in the direction where the discriminator scores are higher and hence leave the region it was converging to. We suggest that this is one of the causes of instability in GAN training. If the regions with high discriminator scores contain only a few modes of the data distribution, this leads to mode collapse; see the study in \S \ref{sec:s3}.

To remedy this situation, we propose to modify the training gradient $\nabla \LL_D$ to encourage high discriminator scores for real data. We propose a new family of discriminator loss functions, which we call adaptive weighted loss function or aw-loss function; see equation (\ref{intro:eq:02}). 

We first show that the proposed weighted discriminator (\ref{intro:eq:02}) with fixed weights 
carries the same theoretical properties of the original GAN as stated in~\cite{NIPS2014_5423, Mescheder2018ICML} for binary-cross-entropy loss function, i.e. when the min-max problem is solved exactly, we recover the data distribution. 

\begin{thm}\label{thm:weightgan} 
    Let $p_{\data} (x)$ and $p_{\model} (x)$ be the density functions for the data and model distributions, $\PP_\data$ and $\PP_\model$, respectively.  Consider $\LL^{aw} (D, p_\model)= w_r \E_{x\sim p_d}\big[\log D(x) \big]+ w_f\E_{x\sim p_g}\big[\log (1-D(x)) \big] $ with fixed $w_r, w_f >0$.
    \begin{enumerate}[leftmargin=*]
        \item Given a fixed $p_\model (x)$,  $\LL^{aw} (D, p_\model)$ is maximized by $D^* (x) = \frac{ w_r p_{\data} (x)}{w_r p_{\data} (x) +w_f p_{\model} (x)}$ for $x \in {\rm supp}(p_\data) \cup {\rm supp}(p_\model)$.
        \item $\min_{p_\model} \max_D \LL^{aw} (D, p_\model) = w_r\log \frac{w_r}{w_r+w_f}+w_f \log \frac{w_f}{w_r+w_f}$ with the minimum attained by $p_{\model} (x) = p_{\data} (x)$.
    \end{enumerate}
\end{thm}

See \ref{Supp_Mater:A} for a proof of Theorem \ref{thm:weightgan}. To choose the weights $w_r$ and $w_f$, we propose an adaptive scheme, where the weights $w_r$ and $w_f$ are determined using gradient information of both $\LL_r$ and $\LL_f$. This structure allows us to adjust the direction of the gradient of the discriminator loss function to achieve the goal of training to increase both $\LL_r$ and $\LL_f$, or at least not to decrease either loss. We propose Algorithm \ref{alg:1} based on the following gradient relations with various weight choices.

\begin{thm}\label{thm:normalized}
    Consider $\LL_D^{aw}$ in (\ref{intro:eq:02}) and the gradient $\nabla \LL_D^{aw}$.
    \begin{enumerate}[leftmargin=*]
        \item If $w_r=\frac{1}{\norm{\nabla \LL_r}_2}$ and $w_f=\frac{1}{\norm{\nabla\LL_f}_2}$, then $\nabla \LL_D^{aw}$ is the angle bisector of $\nabla \LL_r$ and $ \nabla \LL_f$, i.e.
        \begin{align*}
            \angle_2 \left(\nabla \LL_D^{aw}, \nabla \LL_r\right)&=\angle_2 \left(\nabla \LL_D^{aw}, \nabla \LL_f\right)\\
            &=\angle_2 \left(\nabla \LL_r, \nabla \LL_f\right)/2.
        \end{align*}
        \item If $w_r=\frac{1}{\norm{\nabla \LL_r}_2}$ and $w_f=-\frac{\la\nabla\LL_r,\nabla\LL_f\ra_2}{\norm{\nabla \LL_f}_2^2\cdot\norm{\nabla \LL_r}_2}$, then
        \begin{equation*}
            \angle_2\left(\nabla \LL_D^{aw}, \nabla \LL_f\right)=90^{\circ},\; \angle_2\left( \nabla \LL_D^{aw}, \nabla \LL_r\right) \leq 90^{\circ}.
        \end{equation*}
        \item If $w_r=-\frac{\la\nabla\LL_r,\nabla\LL_f\ra_2}{\norm{\nabla \LL_r}_2^2\cdot\norm{\nabla \LL_f}_2}$ and $w_f=\frac{1}{\norm{\nabla\LL_f}_2}$, then
        \begin{equation*}
            \angle_2\left( \LL_D^{aw}, \nabla \LL_r\right)=90^{\circ},\; \angle_2\left( \nabla \LL_D^{aw}, \nabla \LL_f\right) \leq 90^{\circ}.
        \end{equation*}
    \end{enumerate}
\end{thm}

See \ref{Supp_Mater:A} for a proof of Theorem \ref{thm:normalized}. The first case in the above theorem allows us to choose weights for (\ref{intro:eq:02}) such that we can train $\LL_r$ and $\LL_f$  by going in the direction of the angle bisector. However, sometimes the direction of the angle bisector might not be optimal. For example, if the angle between $\nabla \LL_r$ and $\nabla \LL_f$ is close to $180^{\circ}$, then the bisector direction will effectively not train either loss. During training, $\LL_f$ is often easier to train than $\LL_r$ meaning that the fake gradient has a larger magnitude. In this situation, we might want to train just on the real gradient direction by simply choosing $w_f=0$, but if the angle between $\nabla \LL_r$ and $\nabla \LL_f$ is obtuse, we will increase $\LL_r$ but significantly decrease $\LL_f$ which is undesirable. The second case in Theorem \ref{thm:normalized} suggests a direction that forms an acute angle with $\nabla \LL_r$ and orthogonal to $\nabla \LL_f$ (see Figure \ref{fig:case1}); such a direction will increase $\LL_r$ and to the first order approximation will leave $\LL_f$ unchanged. When $\LL_r$ is high, the third case in Theorem \ref{thm:normalized} would allow us to increase $\LL_f$ while minimizing changes to $\LL_r$.

Inspired by the Theorem \ref{thm:normalized} and observations that we have made, we can calibrate discriminator training in a way that produces and maintains high real loss to reduce fluctuations in the real loss (or real discriminator scores) to improve stability. Algorithm \ref{alg:1} describes the procedure for updating weights of the aw-loss function in (\ref{intro:eq:02}) during training using the information of $\nabla \LL_r$ and $\nabla \LL_f$.

\IncMargin{0.15em}\SetNlSty{text}{}{:}
\begin{algorithm}
    \SetAlgoLined
        \textbf{Given:} $\PP_\data$ and $\PP_\model$ - data and model distributions\;
        \textbf{Given:} $\alpha_1=0.5$, $\alpha_2=0.75$, $\varepsilon=0.05$, $\delta=0.05$\;
        \textbf{Sample:} $x_1, \ldots, x_n \sim \PP_\data$ and $y_1, \ldots, y_n \sim \PP_\model$\;
        \textbf{Compute:} $\nabla\LL_r$, $\nabla\LL_f$, $s_r=\frac{1}{n}\sum_{i=1}^n\sigma(D(x_i))$, $s_f=\frac{1}{n}\sum_{j=1}^n\sigma(D(y_j))$\;
        \uIf{$s_r < s_f-\delta$ or $s_r < \alpha_1$}{
            \uIf{$\angle_2\left(\nabla\LL_r,\nabla\LL_f\right) > 90^{\circ}$}{
            $w_r=\frac{1}{\norm{\nabla\LL_r}_2}+\varepsilon$; 
            $w_f=\frac{-\la\nabla\LL_r,\nabla\LL_f\ra_2}{\norm{\nabla \LL_f}_2^2\cdot\norm{\nabla \LL_r}_2}+\varepsilon$;\hspace{-0.1in}
            }
            \Else{
            $w_r=\frac{1}{\norm{\nabla\LL_r}_2}+\varepsilon$; $w_f=\varepsilon$\;
            }
        }
        \uElseIf{$s_r > s_f-\delta$ and $s_r > \alpha_2$}{
            \uIf{$\angle_2\left(\nabla\LL_r,\nabla\LL_f\right) > 90^{\circ}$}{
            $w_r=\frac{-\la\nabla\LL_r,\nabla\LL_f\ra_2}{\norm{\nabla \LL_r}_2^2\cdot\norm{\nabla \LL_f}_2}+\varepsilon$; $w_f=\frac{1}{\norm{\nabla\LL_f}_2}+\varepsilon$;\hspace{-0.1in}
            }
            \Else{
            $w_r=\varepsilon$; $w_f=\frac{1}{\norm{\nabla\LL_f}_2}+\varepsilon$\;
            }
        }
        \Else{
        $w_r=\frac{1}{\norm{\nabla\LL_r}_2}+\varepsilon$;  $w_f=\frac{1}{\norm{\nabla\LL_f}_2}+\varepsilon$\;
        }
    \caption{Adaptive weighted discriminator method for one step of discriminator training.}
    \label{alg:1}
\end{algorithm}

Algorithm \ref{alg:1} is designed to first avoid, up to the first order approximation, decreasing $\LL_r$ or  $\LL_f$ during a gradient ascent iteration. Furthermore, it chooses to favor training real 
loss unless the mean real score is greater than the mean fake score (i.e. $s_f \le s_r$) and the real mean score is at least $\alpha_1=0.5$ (i.e. $\alpha_1 \le s_r$). Here the mean discriminator scores $s_r$ and $s_f$ represent the mean probability that the discriminator assigns to $x_i$'s and  $y_j$'s respectively as real data. 
When $s_r$ is highly satisfactory with $ s_r \ge \alpha_2=0.75$ (the midpoint between the minimum probability 0.5 and the maximum probability 1 for correct classifications of real data), we favor training the fake loss; otherwise, we train both equally. By maintaining these training criteria, we will reduce the fluctuations in real and fake discriminator scores and hence avoid instability. See study in \S\ref{sec:s3}. Note that  we impose a small gap $\delta=0.05$ in $s_f-\delta > s_r$ to account for situations when $s_r$ is nearly identical to $s_f$. 

\begin{figure}[t]
    \begin{center}
        \includegraphics[width=0.215\textwidth]{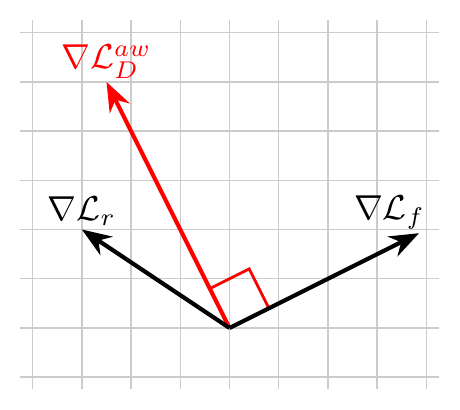}
    \end{center}
    \caption{Depiction of the second case of Theorem \ref{thm:normalized}.}
    \label{fig:case1}
\end{figure}

The way we favor training the real or fake loss depends on whether the angle between $\nabla \LL_r$ and $\nabla \LL_f$ is obtuse or not. In Algorithm \ref{alg:1}, the first and the third cases are concerned with the more frequent situation (see \S\ref{sec:s1} and Figure \ref{fig:angles_orig_vs_aw}) where the angle between $\nabla \LL_r$ and $\nabla \LL_f$ is obtuse. These cases are the ones that are developed in Theorem \ref{thm:normalized}. In the first case, we favor training real loss by going in the direction orthogonal to the $\nabla \LL_f$, illustrated in Figure \ref{fig:case1}. In the third case, we favor the fake loss by going in the direction orthogonal to $\nabla \LL_r$. In a similar manner, the second and the fourth cases are concerned with the situation when the angle between $\nabla \LL_r$ and $\nabla \LL_f$ is acute. We use the same criteria to decide if training should favor the real or fake directions, but in this case we favor training the real or fake loss by using the direction of the corresponding gradient. Lastly, in the fifth case, it is desirable to increase both $s_r$ and $s_f$ without either taking priority, so we choose to train in the direction of the angle bisector between $\nabla \LL_r$ and $\nabla \LL_f$. 

The two threshold  $\alpha_1$ and $\alpha_2$ in Algorithm \ref{alg:1}  can be treated as hyperparameters. Our ablation studies show that the default $\alpha_1=0.5$ and $\alpha_2=0.75$ as discussed earlier are indeed good choices, see \ref{Supp_Mater:B}.

\begin{figure*}
    \begin{center}
        \includegraphics[scale=.552]{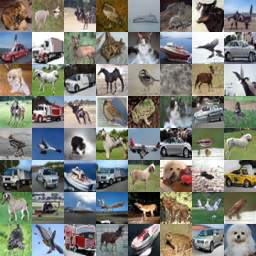}
        \includegraphics[scale=0.368]{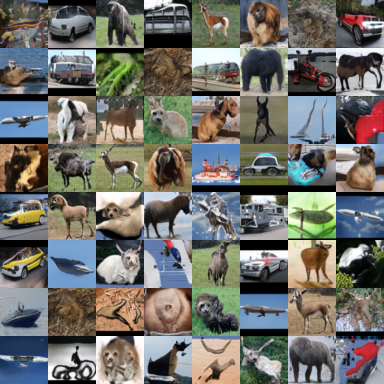}
        \includegraphics[scale=.552]{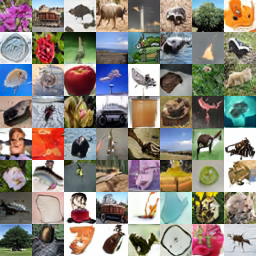}
    \end{center}
    \caption{aw-AutoGAN: CIFAR-10 (left), STL-10 (center), CIFAR-100 (right); samples randomly generated.}\label{fig:sample_images}
\end{figure*}

All weights stated in Algorithm \ref{alg:1} normalize both the real and fake gradients for the purpose of avoiding differently sized gradients, which has the effect of preventing exploding gradients and speeds up training, i.e. achieves better IS and FID with fewer epochs, see Figure \ref{fig:IS_FID_40epochs}. With this implementation, we implement a linear learning rate decay to ensure convergence. However, aw-method performs well without normalization, and achieves comparable results. We list the detailed results in \ref{Supp_Mater:B}.

A small constant $\varepsilon$ is added to all the weights in Algorithm \ref{alg:1} to avoid numerical discrepancies in cases that would prevent the discriminator model from training/updating. As an example, there are cases when our algorithm would set $w_r=0$ but at the same time $\nabla \LL_f$ would be almost zero, which will result in $\nabla \LL_D^{aw}$ being practically zero. We have set $\varepsilon=0.05$ in all of our experiments.

Algorithm \ref{alg:1} has a small computational overhead. At each iteration we compute inner products and norms that are used for computing $w_r$ and $w_f$, and then use these weight to update $\nabla \LL_D^{aw}$. If we have $k$ trainable parameters, then it takes an order of $6k$ operations to compute inner products between real--fake, real--real, and fake--fake gradients, and an order of $3k$ operations to form $\nabla \LL_D^{aw}$, totalling to an order of $9k$ operations in Algorithm \ref{alg:1}. This is a fraction of total computational complexity for one training iteration.

\section{Experiments \& Results}\label{sect:experiments}
\begin{figure*}
    \begin{center}
        \includegraphics[width=0.305\textwidth]{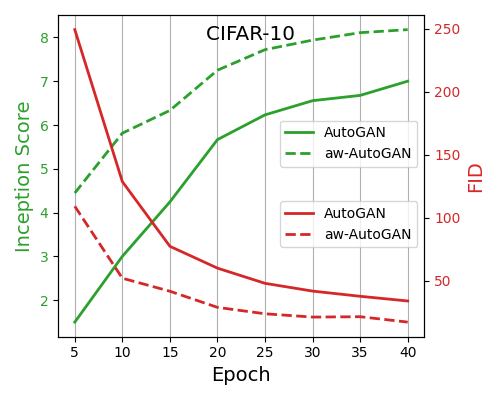}
        \includegraphics[width=0.305\textwidth]{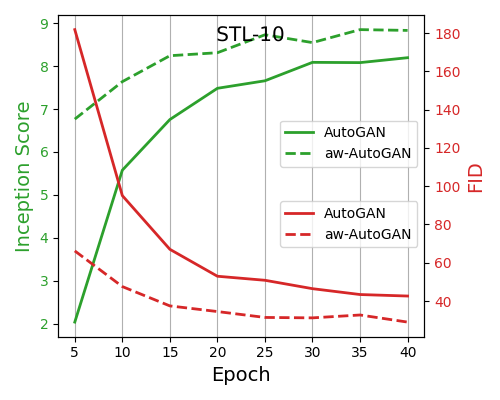}
        \includegraphics[width=0.305\textwidth]{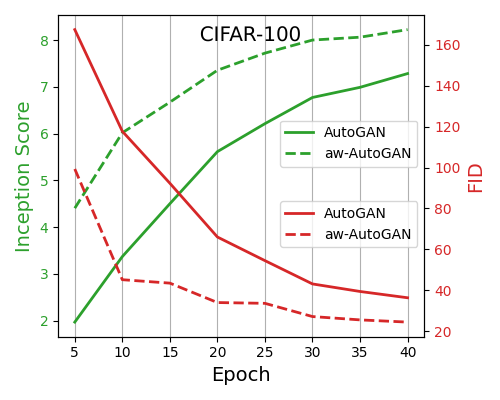}
    \end{center}
    \caption{AutoGAN vs aw-AutoGAN IS and FID plots for the first 40 epochs.}
    \label{fig:IS_FID_40epochs}
\end{figure*}

We implement our Adaptive Weighted Discriminator for SN-GAN~\cite{miyato2018spectral} and AutoGAN~\cite{Gong_2019_ICCV} models, and for SN-GAN~\cite{miyato2018spectral} and BigGAN~\cite{brock2018large} models, on unconditional and conditional image generating tasks, respectively (commonly referred to as unconditional and conditional GANs). AutoGAN is an architecture based on neural search. In our experiments; we do not invoke a neural search with our aw-loss, we have simply implemented the aw-method on the model and architecture exactly provided by~\cite{Gong_2019_ICCV}. 

We test our method on three datasets: \emph{CIFAR-10}~\cite{Krizhevsky09learningmultiple}, \emph{STL-10}~\cite{pmlr-v15-coates11a}, and \emph{CIFAR-100}~\cite{Krizhevsky09learningmultiple}. The datasets and implementation details are provided in \ref{Supp_Mater:B}.
We present the implementation with normalized gradients using linear learning rate decay. We also give results of non-normalized version without learning rate decay in \ref{Supp_Mater:B}.

All of the above mentioned models train the discriminator by minimizing the negative hinge loss ~\cite{lim2017geometric, tran2017hierarchical}. Our aw-loss also uses  the negative hinge loss as follows: 
\begin{align}
    \LL_D^{aw}=& - w_r\cdot\E_{x\sim p_d}\big[\min(0,D(x)-1)\big] \nonumber\\
    & \quad - w_f\cdot\E_{z\sim p_z}\big[\min(0,-1-D(G(z))\big],
\end{align}
with $w_r$ and $w_f$ updated every iteration using Algorithm \ref{alg:1}.

To evaluate the performance of the models, we employ the widely used Inception Score~\cite{IS_NIPS2016_6125} (IS) and Fr\'{e}chet Inception Distance~\cite{FID_NIPS2017_7240} (FID) metrics; see~\cite{GANs_equal} for more details. 
We compute these metrics every 5 epochs and we report the best IS and FID achieved by each model within the 320 (SN-GAN), 300 (AutoGAN), and BigGAN (1,000) training epochs as in the corresponding original works.

We first present the results  for the unconditional GAN  for the datasets CIFAR-10, STL-10 and CIFAR-100 in Tables \ref{cifar10_results}, \ref{stl10_results}, and \ref{cifar100_results}  respectively. In addition to baseline results, we have included top published results for each dataset for comparison purposes. 

\begin{table}[H]
    \begin{center}
        \begin{tabular}{l||c|c}
            \hline
            \rowcolor{lightgray}Method & IS $\uparrow$ & FID $\downarrow$ \\
            \hline
            Imp. MMD GAN \cite{wang2018improving} & 8.29 & 16.21 \\
            \hline
            MGAN \cite{hoang2018mgan} & 8.33{\scriptsize$\pm$.10} & 26.7 \\
            \hline
            MSGAN \cite{MSGAN} & - & 11.40 \\
            \hline
            SRN-GAN \cite{sanyal2019stable} & 8.53{\scriptsize$\pm$.04} & 19.57 \\
            \hline
            StyleGAN2~\cite{Karras2019stylegan2} & \textbf{9.21{\scriptsize$\pm$.09}} & \textbf{8.32} \\
            \hline\hline
            SN-GAN \cite{miyato2018spectral} & 8.22{\scriptsize$\pm$.05} & 21.7 \\
            \hline
            aw-SN-GAN (Ours) & 8.53{\scriptsize$\pm$.11} & 12.32 \\
            \hline\hline
            AutoGAN \cite{Gong_2019_ICCV} & 8.55{\scriptsize$\pm$.10} & 12.42 \\
            \hline
            aw-AutoGAN (Ours) & 9.01{\scriptsize$\pm$.03} & 11.82 \\
        \end{tabular}
    \end{center}
    \caption{Unconditional GAN: CIFAR-10.}
    \label{cifar10_results}
\end{table}

For CIFAR-10 in Table \ref{cifar10_results}, our methods significantly improve the baseline results. Indeed, our aw-AutoGAN achieves IS substantially above all comparisons other than StyleGAN2. StyleGAN2 outperforms aw-AutoGAN but uses 26.2M parameters vs. 5.4M for aw-AutoGAN.

\begin{table}[H]
    \begin{center}
        \begin{tabular}{l||c|c}
            \hline
            \rowcolor{lightgray}Method & IS $\uparrow$ & FID $\downarrow$ \\
            \hline
            ProbGAN \cite{he2018probgan} & 8.87{\scriptsize$\pm$.09} & 46.74 \\
            \hline
            Imp. MMD GAN \cite{wang2018improving} & 9.34 & 37.63 \\
            \hline
            MGAN \cite{hoang2018mgan} & 9.22{\scriptsize$\pm$.11} & - \\
            \hline
            MSGAN \cite{MSGAN} & - & 27.10 \\
            \hline\hline
            SN-GAN \cite{miyato2018spectral} & 9.10{\scriptsize$\pm$.04} & 40.10 \\
            \hline
            aw-SN-GAN (Ours) & \textbf{9.53{\scriptsize$\pm$.10}} & 36.41 \\
            \hline\hline
            AutoGAN \cite{Gong_2019_ICCV} & 9.16{\scriptsize$\pm$.12} & 31.01 \\
            \hline
            aw-AutoGAN (Ours) & 9.41{\scriptsize$\pm$.09} & \textbf{26.32} \\
        \end{tabular}
    \end{center}
    \caption{Unconditional GAN: STL-10.
    }
    \label{stl10_results}
\end{table}

For STL-10 in Table \ref{stl10_results}, our methods also significantly improve SN-GAN and AutoGAN baseline results. Our aw-SN-GAN achieved the highest IS and aw-AutoGAN achieved the lowest FID score among comparisons.

\begin{table}[H]
    \begin{center}
        \begin{tabular}{l||c|c}
            \hline
            \rowcolor{lightgray}Method & IS $\uparrow$ & FID $\downarrow$ \\
            \hline
            SS-GAN \cite{SSGAN} & - & 21.02$^{\dag}$\\ 
            \hline
            MSGAN \cite{MSGAN} & - & 19.74\\
            \hline
            SRN-GAN \cite{sanyal2019stable} & 8.85 & 19.55\\
            \hline\hline
            SN-GAN \cite{miyato2018spectral} & 8.18{\scriptsize$\pm$.12}$^{*}$ & 22.40$^{*}$\\
            \hline
            aw-SN-GAN (Ours) & 8.31{\scriptsize$\pm$.02} & 19.08\\
            \hline\hline
            AutoGAN \cite{Gong_2019_ICCV} & 8.54{\scriptsize$\pm$.10}$^{*}$ & 19.98$^{*}$\\
            \hline
            aw-AutoGAN (Ours) & \textbf{8.90{\scriptsize$\pm$.06}} & \textbf{19.00}\\
        \end{tabular}
    \end{center}
    \caption{Unconditional GAN: CIFAR-100; $^{*}$ - results from our test; $^{\dag}$ - quoted from \cite{MSGAN}.}
    \label{cifar100_results}
\end{table}

For CIFAR-100 in Table \ref{cifar100_results}, our methods improve IS significantly for AutoGAN but modestly for SN-GAN. Our aw-Auto-GAN achieved the highest IS and the lowest FID score among comparisons. 

We have also included some visual examples that were randomly generated by our aw-Auto-GAN model in Figure \ref{fig:sample_images}. We also consider the convergence of our method against training epochs by  plotting in Figure \ref{fig:IS_FID_40epochs} the IS and FID scores of 50,000 generated samples at every 5 epochs  for AutoGAN vs aw-AutoGAN. For all the datasets, our model consistently achieves faster convergence than the baseline. 

We next consider our aw-method for a class conditional image generating task using two base models, SN-GAN~\cite{miyato2018spectral} and BigGAN~\cite{brock2018large}, on CIFAR-10 and CIFAR-100 datasets. Results are listed in Table \ref{table:conditional_gan}.

\begin{table}[H]
    \begin{center}
        \begin{tabular}{l||c|c||c|c}
            & \multicolumn{2}{c||}{CIFAR-10} & \multicolumn{2}{c}{CIFAR-100}\\
            \hline
            \rowcolor{lightgray}Method & IS $\uparrow$ & FID $\downarrow$ & IS $\uparrow$ & FID $\downarrow$ \\
            \hline
            cGAN~\cite{miyato2018cgans} & 8.62 & 17.5 & 9.04 & 23.2\\
            \hline
            MHinge~\cite{Kavalerov_2021_WACV} & \textbf{9.58{\scriptsize$\pm$.09}} & 7.50 & \textbf{14.36{\scriptsize$\pm$.17}} & 17.3 \\
            \hline\hline
            SNGAN~\cite{miyato2018spectral} & 8.60{\scriptsize$\pm$.08} & 17.5 & 9.30$^{\dag}$ & 15.6$^{\dag}$\\
            \hline
            aw-SNGAN & 9.03{\scriptsize$\pm$.11} & 8.11 & 9.48{\scriptsize$\pm$.13} & 14.42\\
            \hline\hline
            BigGAN~\cite{brock2018large} & 9.22 & 14.73 & 10.99{\scriptsize$\pm$.14} & 11.73 \\
            \hline
            aw-BigGAN & 9.52{\scriptsize$\pm$.10} & \textbf{7.03} & 11.22{\scriptsize$\pm$.17} & \textbf{10.23} \\
        \end{tabular}
    \end{center}
    \caption{Conditional GAN: CIFAR-10 and CIFAR-100; $^{\dag}$ - quoted from \cite{shmelkov2018good}; BigGAN~\cite{brock2018large} CIFAR-100 results based on our test using the code from~\cite{BG_github}.}
    \label{table:conditional_gan}
\end{table}

Table \ref{table:conditional_gan} shows that our method works well for the conditional GAN too. The aw-method   significantly improves the SN-GAN and BigGAN baselines. Indeed, our aw-BigGAN achieved the best FID for both CIFAR-10 and CIFAR-100 among comparisons.

\section{Exploratory \& Ablation Studies}\label{sec:ablation_n_exploratory_studies}
In this section, we present three studies to illustrate potential problems of equally weighted GAN loss and advantages of our adaptive weighted loss. The hinge loss is implemented in the first and second studies, and a binary cross-entropy loss function is used for the third.

\subsection{Angles between Gradients} \label{sec:s1}
\begin{figure}[H]
    \begin{center}
        \includegraphics[width=0.4775\textwidth]{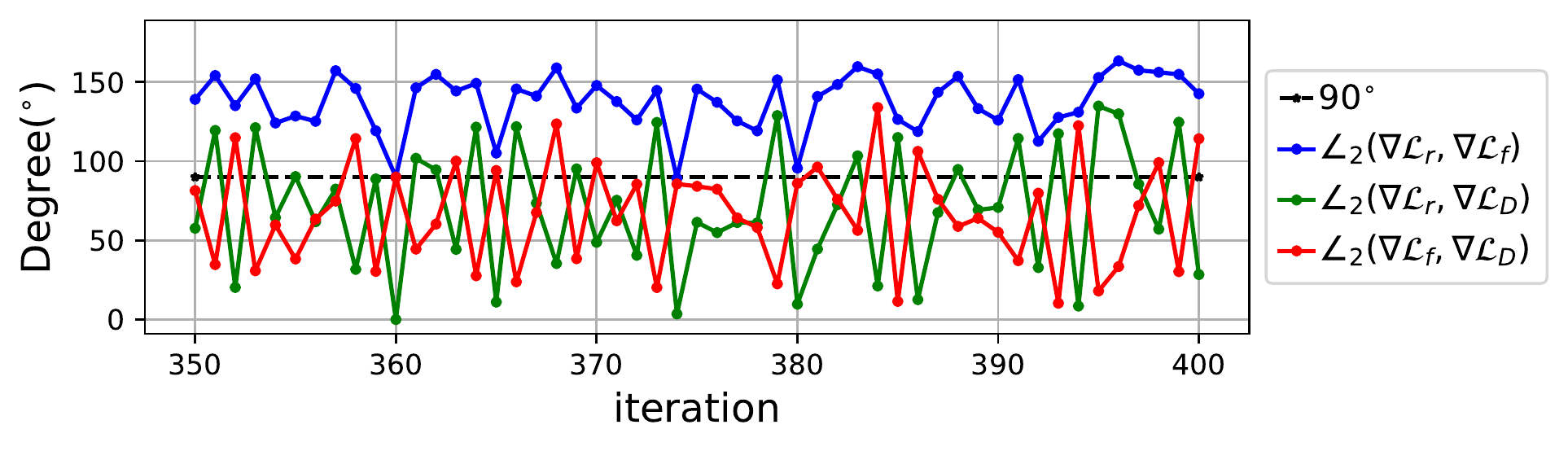}
        \includegraphics[width=0.4775\textwidth]{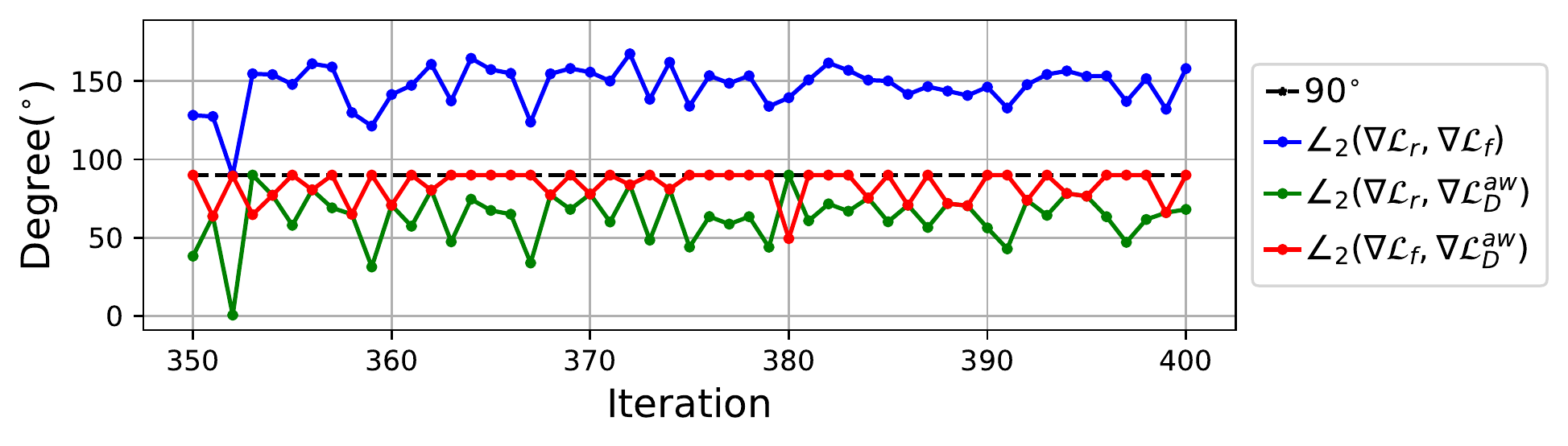}
    \end{center}
    \caption{Angles between gradients at each iteration. Top: original loss; Bottom: aw-loss.}
    \label{fig:angles_orig_vs_aw}
\end{figure}

In the first study, we examine the angles between $\nabla \LL_r$, $\nabla \LL_f$, $\nabla \LL_D$ (or $\nabla \LL_D^{aw}$). We use the CIFAR-10 dataset with the DCGAN architecture \cite{radford2015unsupervised} and we look at 50 iterations in the first epoch of training. In Figure \ref{fig:angles_orig_vs_aw}, we plot the following 3 angles: $\angle_2(\nabla \LL_r, \nabla \LL_f)$, $\angle_2(\nabla \LL_r, \nabla \LL_D)$ and $\angle_2(\nabla \LL_f, \nabla \LL_D)$ against iterations for the original loss $\LL_D$ (\ref{intro:eq:01}) on the top and for the aw-loss $\LL_D^{aw}$ on the bottom. For the original loss (Left), $\angle_2(\nabla \LL_r, \nabla \LL_f)$ (blue) stays greater then $90^{\circ}$, closer to $180^{\circ}$. $\angle_2(\nabla \LL_r, \nabla \LL_D)$ (green) often goes above $90^{\circ}$ and so the training is often done to decrease the real loss. $\angle_2(\nabla \LL_f, \nabla \LL_D)$ also goes above $90^{\circ}$, though to a lesser extent. With the aw-loss (Right), $\angle_2(\nabla \LL_r, \nabla \LL_D^{aw})$ and $\angle_2(\nabla \LL_f, \nabla \LL_D^{aw})$ stay below the $90^{\circ}$ line and indicate that we train in the direction of $\nabla \LL_r$ and orthogonal to $\nabla \LL_f$ in most steps.

\subsection{Real Discriminator Scores and Real-Fake Gap after Training} \label{sec:s2}
\begin{figure}[t]
    \begin{center}
        \includegraphics[width=0.45\textwidth]{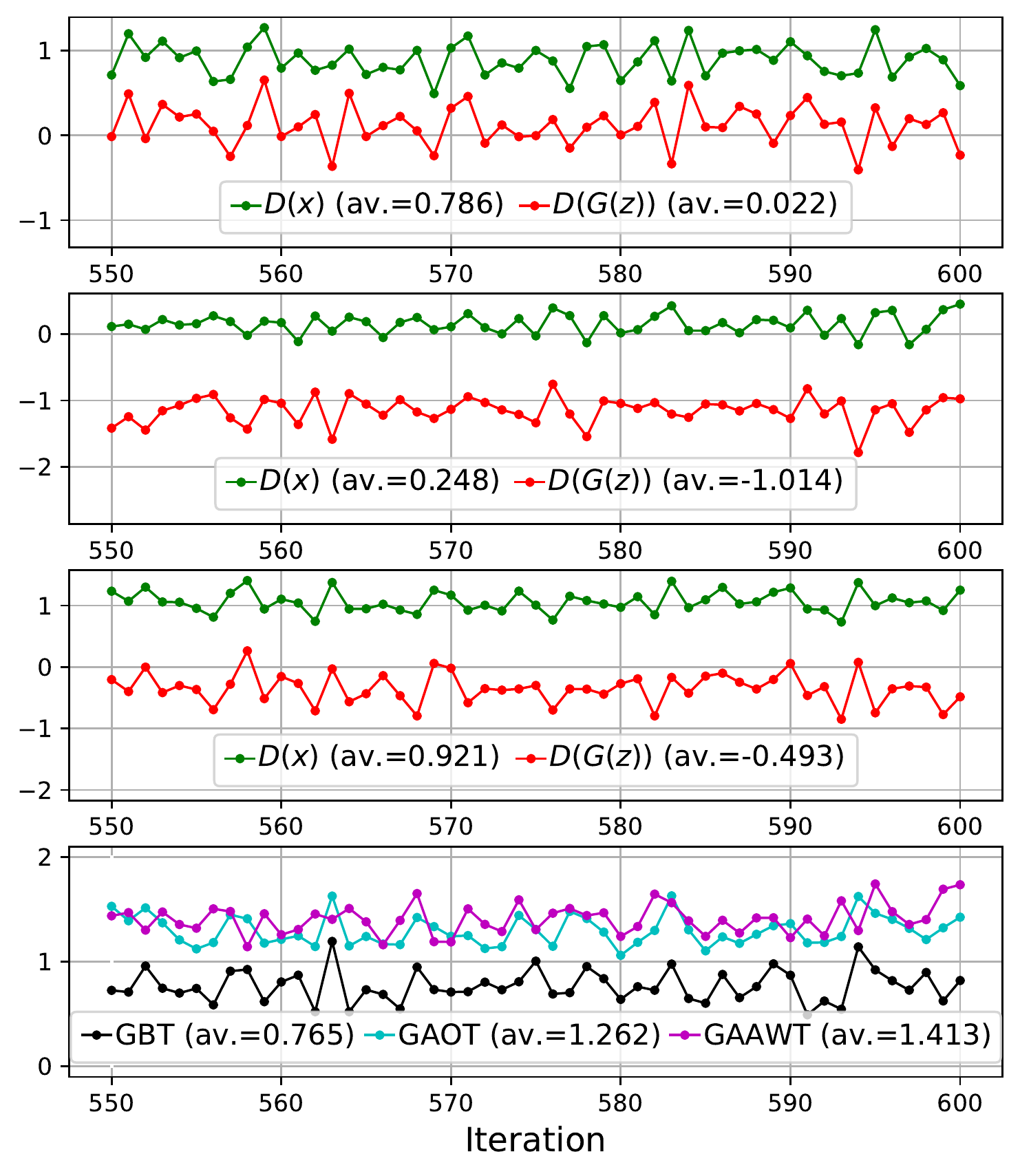}
    \end{center}
    \caption{Mean discriminator scores for real data $D(x)$ and fake data $D(G(z))$ (Row 1: before training, Row 2: after original training with $\LL_D$, Row 3: after training with aw-loss  $\LL_D^{aw}$) and their gap (Row 4: GBT - gap before training; GAOT - gap after original training; GAAWT - gap after aw-loss training)}
    \label{fig:ablation_study}
\end{figure}

\begin{figure*}[t]
    \includegraphics[width=\textwidth]{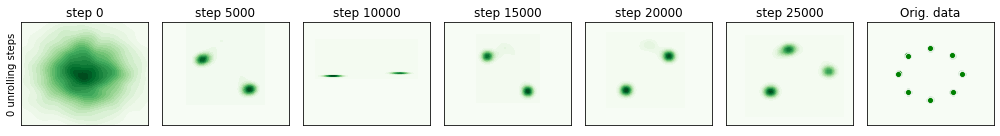}
    
    \includegraphics[width=.8475\textwidth]{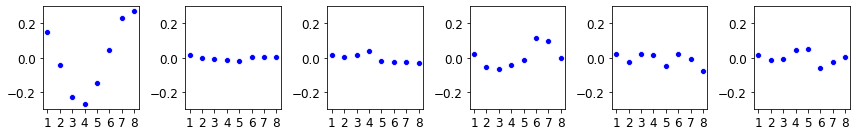}
    
    \includegraphics[width=\textwidth]{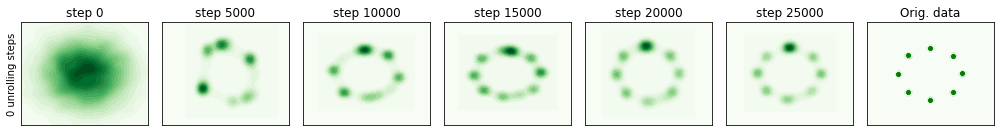}
    
    \includegraphics[width=.8575\textwidth]{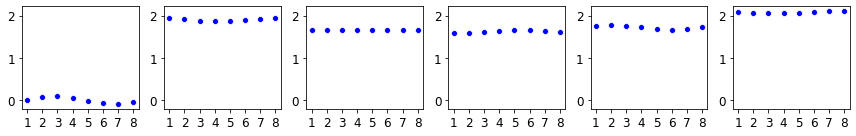}
    \caption{Mixture of eight 2D Gaussian distributions centered at 8 points (right-most column). Row 1: GAN sample points produced by generators; Row 2: GAN mean discriminator scores for each of 8 classes; 
    Row 3: aw-GAN sample points produced by generators; Row 4: aw-GAN mean discriminator scores for each of eight classes;}
    \label{fig:mode_collapse_study}
\end{figure*}

Our second experiment is an ablation study to show that aw-loss increases the discriminator scores for real data and increases the gap between real and fake discriminator scores. We again apply the DCGAN model with the original loss $\LL_D$ to CIFAR-10 and at every iteration we examine the mean discriminator score for the mini-batch of the real set and the mean discriminator scores for the mini-batch of the fake dataset generated by the generator. We use the logit output of the discriminator network as the score. We plot these two mean scores against each iteration before training in the first row of Figure \ref{fig:ablation_study} and after training (with the original loss $\LL_D$) in the second row. At each of the above training iterations, we replace $\LL_D$ by the aw-loss $\LL_D^{aw}$ (\ref{intro:eq:02}) and train for one iteration with the same training mini-batch. We plot the mean discriminator scores for the mini-batches of the real and fake dataset after this training in the third row of Figure \ref{fig:ablation_study}. We further present the gaps between the two scores before training and after training using the original loss and using aw-loss in the fourth row of Figure \ref{fig:ablation_study}.

Figure \ref{fig:ablation_study} shows that training with aw-loss leads to higher real discriminator scores (0.921 epoch average) than training with the original loss (0.248 epoch average). The average gap between real and fake scores is also larger with the aw-loss at 1.413 vs. 1.262 of the original loss. Therefore, with the same model and the same training mini-batch, the aw-loss produces higher discriminator scores for the real dataset and larger gaps between real and fake scores. These are two important properties of a discriminator for the generator training. 

\subsection{Instability and Real Discriminator Scores} \label{sec:s3}

Our third study examines benefits of high discriminator scores for a real dataset with respect to instability and mode collapse of GAN training. We use a synthetic dataset with a mixture of Gaussian distributions used to test unrolled GAN in \cite{Unrolled_GAN}. The dataset consists of eight 2D Gaussian distributions centered at eight equally distanced points on the unit circle. We train with a plain GAN as in \cite{Unrolled_GAN} and we plot samples of (fake) data generated by the generator every 5,000 iterations on the first row of of Figure \ref{fig:mode_collapse_study}. We see that the generated data converging to two or three points but then moving off, demonstrating instability and mode collapse. To understand this phenomenon, at each of the iterations that we study in Figure \ref{fig:mode_collapse_study}, we generate 100 (real) data points from each of the eight classes and compute their mean discriminator scores (as the logit output of the discriminator). We plot the mean scores against the classes in the second row of Figure \ref{fig:mode_collapse_study}. We observe that the discriminator scores for the real data do not increase much during training, staying around 0, which corresponds to 0.5 probability after the logistic sigmoid function. The scores are also uneven among different classes. We believe these cause the instability in the generator training.

We compare the GAN results with aw-GAN that applies our adaptive weighted discriminator to the plain GAN and we present the corresponding plots of generated data points (fake) in the third row of Figure \ref{fig:mode_collapse_study} and the corresponding discriminator scores on the eight classes in the bottom row. In this case, the generator gradually converges to all eight classes and the discriminator scores stay high for all eight classes. Even though the generator was starting to converge to a few classes (step 5,000), the discriminator scores remain high for all classes. Then the generator continues to converge while convergence to other classes occurs. We believe the high real discriminator scores maintains stability and prevents mode collapse in this case. 

\section*{Conclusions}
This paper pinpoints potential negative effects of the traditional GAN training on the real loss (and fake loss) and points out that this is a potential cause of instability and mode collapse. To remedy these issues, we have proposed the Adaptive Weighted Discriminator method to increase and maintain high real loss. Our experiments demonstrate the benefits and the competitiveness of this method.

\section*{Acknowledgements}
We would like to thank Devin Willmott and Xin Xing for their initial efforts and helpful advice. We thank Rebecca Calvert for reading the manuscript and providing us with many valuable comments/suggestions. We also thank the University of Kentucky Center for Computational Sciences and Information Technology Services Research Computing for their support and use of the Lipscomb Compute Cluster.

{\small
\bibliographystyle{ieee_fullname}
\bibliography{egbib.bib}
}
\newpage

\onecolumn
\appendix
\addcontentsline{toc}{section}{Supplementary Materials}
\renewcommand{\thesubsection}{Appendix \Alph{subsection}}
\section*{Supplementary Materials}
\subsection{}\label{Supp_Mater:A}
\renewcommand{\thesubsubsection}{\Alph{subsection}.\arabic{subsubsection}}
\subsubsection{Proof of Theorem \ref{thm:weightgan}}
\setcounter{thm}{0}
\begin{thm}
    Let $p_{\data} (x)$ and $p_{\model} (x)$ be the density functions for the data and model distributions, $\PP_\data$ and $\PP_\model$, respectively.  Consider $\LL^{aw} (D, p_\model)= w_r \E_{x\sim p_d}\big[\log D(x) \big]+ w_f\E_{x\sim p_g}\big[\log (1-D(x)) \big] $ with fixed $w_r, w_f >0$.
    \begin{enumerate}[leftmargin=*]
        \item Given a fixed $p_\model (x)$,  $\LL^{aw} (D, p_\model)$ is maximized by $D^* (x) = \frac{ w_r p_{\data} (x)}{w_r p_{\data} (x) +w_f p_{\model} (x)}$ for $x \in {\rm supp}(p_\data) \cup {\rm supp}(p_\model)$.
        \item $\min_{p_\model} \max_D \LL^{aw} (D, p_\model) = w_r\log \frac{w_r}{w_r+w_f}+w_f \log \frac{w_f}{w_r+w_f}$ with the minimum attained by $p_{\model} (x) = p_{\data} (x)$.
    \end{enumerate}
\end{thm}

\begin{proof}
    \hfill
    \begin{enumerate}[leftmargin=*]
        \item First, the function $f(t)=a\log t+b \log (1-t)$ has its maximum in $[0,1]$ at $t=\frac{a}{a+b}$. 
        Given a fixed $p_\model(x)$, $w_r>0$ and $w_f>0$.
        \begin{align}
            \LL^{aw} (D, p_\model)&=w_r\E_{x\sim p_\data} \left[\log \left(D(x)\right)\right] + w_f\E_{x\sim p_\model} \left[\log \left(1-D(x)\right)\right]\\
            &=\int_x w_rp_\data(x)\log\left(D(x)\right)+w_f p_\model(x)\log \left(1-D(x)\right)dx \\
            &\le \int_x w_rp_\data(x)\log\left(D^*(x)\right)+w_f p_\model(x)\log \left(1-D^*(x)\right)dx \\
             &= w_r\E_{x\sim p_\data} \left[\log \left(\frac{w_rp_\data(x)}{w_rp_\data(x)+w_fp_\model(x)}\right)\right] + w_f\E_{x\sim p_\model} \left[\log\left(\frac{w_fp_\model(x)}{w_rp_\data(x)+w_fp_\model(x)}\right)\right]. 
        \end{align}
        where the equality holds if $D(x)=D^*(x)$. Therefore, $ \LL^{aw} (D, p_\model)$ is maximum when $D=D^*$. 
        
        \item If $p_\model(x)=p_\data(x)$, then $D^*(x)=\frac{w_r}{w_r+w_f}$ and
        \begin{align}
            \max_D \LL^{aw} (D, p_\model)&= w_r\E_{x\sim p_\data}\left[\log\left(\frac{w_r}{w_r+w_f}\right)\right]+w_f\E_{x\sim p_\model}\left[\log\left(\frac{w_f}{w_r+w_f}\right)\right]\\
            &=w_r\log\left(\frac{w_r}{w_r+w_f}\right)+w_f\log\left(\frac{w_f}{w_r+w_f}\right).
        \end{align}
        
        On the other hand,
        \begin{align}
             \max_D \LL^{aw} (D, p_\model)&= w_r\E_{x\sim p_\data} \left[\log \left(\frac{w_rp_\data(x)}{w_rp_\data(x)+w_fp_\model(x)}\right)\right] + w_f\E_{x\sim p_\model} \log\left[ \left(\frac{w_fp_\model(x)}{w_rp_\data(x)+w_fp_\model(x)}\right)\right]\\
            &=w_r\log\left(\frac{w_r}{w_r+w_f}\right)+w_f\log\left(\frac{w_f}{w_r+w_f}\right)\nonumber\\
            &\quad +w_rKL\left(p_\data \bigg| \frac{w_rp_\data+w_fp_\model}{w_r+w_f} \right)+w_fKL\left(p_\model \bigg| \frac{w_rp_\data+w_fp_\model}{w_r+w_f} \right)\\
            &\geq w_r\log\left(\frac{w_r}{w_r+w_f}\right)+w_f\log\left(\frac{w_f}{w_r+w_f}\right),
        \end{align}
        where KL is the Kullback-Leibler divergence and equality holds when $p_\data=\frac{w_rp_\data+w_fp_\model}{w_r+w_f}$ and $p_\model=\frac{w_rp_\data+w_fp_\model}{w_r+w_f}$. Thus we have shown that
        \begin{align}
            \min_{p_\model}\max_D \LL^{aw} (D, p_\model)= w_r\log\left(\frac{w_r}{w_r+w_f}\right)+w_f\log\left(\frac{w_f}{w_r+w_f}\right).
        \end{align}
        and minimum is attained when $p_\model = p_\data$.
    \end{enumerate}
\end{proof}

\subsubsection{Proof of Theorem \ref{thm:normalized}}
\begin{thm}
    Consider $\LL_D^{aw}$ in (\ref{intro:eq:02}) and the gradient $\nabla \LL_D^{aw}$.
    \begin{enumerate}[leftmargin=*]
        \item If $w_r=\frac{1}{\norm{\nabla \LL_r}_2}$ and $w_f=\frac{1}{\norm{\nabla\LL_f}_2}$, then $\nabla \LL_D^{aw}$ is the angle bisector of $\nabla \LL_r$ and $ \nabla \LL_f$, i.e.
        \begin{align}
            \angle_2 \left(\nabla \LL_D^{aw}, \nabla \LL_r\right)&=\angle_2 \left(\nabla \LL_D^{aw}, \nabla \LL_f\right)=\angle_2 \left(\nabla \LL_r, \nabla \LL_f\right)/2.
        \end{align}
        \item If $w_r=\frac{1}{\norm{\nabla \LL_r}_2}$ and $w_f=-\frac{\la\nabla\LL_r,\nabla\LL_f\ra_2}{\norm{\nabla \LL_f}_2^2\cdot\norm{\nabla \LL_r}_2}$, then
        \begin{equation}
            \angle_2\left(\nabla \LL_D^{aw}, \nabla \LL_f\right)=90^{\circ},\; \angle_2\left( \nabla \LL_D^{aw}, \nabla \LL_r\right) \leq 90^{\circ}.
        \end{equation}
        \item If $w_r=-\frac{\la\nabla\LL_r,\nabla\LL_f\ra_2}{\norm{\nabla \LL_r}_2^2\cdot\norm{\nabla \LL_f}_2}$ and $w_f=\frac{1}{\norm{\nabla\LL_f}_2}$, then
        \begin{equation}
            \angle_2\left( \LL_D^{aw}, \nabla \LL_r\right)=90^{\circ},\; \angle_2\left( \nabla \LL_D^{aw}, \nabla \LL_f\right) \leq 90^{\circ}.
        \end{equation}
    \end{enumerate}
\end{thm}
\begin{proof}
    \hfill
    \begin{enumerate}[leftmargin=*]
        \item If $w_r=\frac{1}{\norm{\nabla \LL_r}_2}$ and $w_f=\frac{1}{\norm{\nabla \LL_f}_2}$, then
        \begin{equation}
            \LL_D^{aw}=\frac{1}{\norm{\nabla \LL_r}_2}\LL_r+\frac{1}{\norm{\nabla \LL_f}_2}\LL_f.
        \end{equation}
        Using the definition of Euclidean inner product,
        \begin{align}
            \cos\left(\angle_2 \left(\nabla \LL_D^{aw}, \nabla \LL_r\right)\right)&=\frac{\la \nabla\LL_D^{aw}, \nabla\LL_r\ra_2}{\norm{\nabla\LL_D^{aw}}_2\norm{\nabla\LL_r}_2}\\
            &=\frac{\frac{1}{\norm{\nabla \LL_r}_2}\la \nabla\LL_r, \nabla\LL_r\ra_2+\frac{1}{\norm{\nabla \LL_f}_2}\la \nabla\LL_f, \nabla\LL_r\ra_2}{\norm{\nabla\LL_D^{aw}}_2\norm{\nabla\LL_r}_2} \\
            &=\frac{1}{\norm{\nabla\LL_D^{aw}}_2}+\frac{\la \nabla\LL_r, \nabla\LL_f\ra_2}{\norm{\nabla\LL_D^{aw}}_2\norm{\nabla\LL_r}_2\norm{\nabla \LL_f}_2} \label{supp_mat:pf1:c1:eq1}
        \end{align}
        \begin{align}
            \cos\left( \angle_2 \left(\nabla \LL_D^{aw}, \nabla \LL_f\right)\right)&=\frac{\la \nabla\LL_D^{aw}, \nabla\LL_f\ra_2}{\norm{\nabla\LL_D^{aw}}_2\norm{\nabla\LL_f}_2}\\
            &=\frac{\frac{1}{\norm{\nabla \LL_r}_2}\la \nabla\LL_r, \nabla\LL_f\ra_2+\frac{1}{\norm{\nabla \LL_f}_2}\la \nabla\LL_f, \nabla\LL_f\ra_2}{\norm{\nabla\LL_D^{aw}}_2\norm{\nabla\LL_f}_2}\\
            &=\frac{1}{\norm{\nabla\LL_D^{aw}}_2}+\frac{\la \nabla\LL_r, \nabla\LL_f\ra_2}{\norm{\nabla\LL_D^{aw}}_2\norm{\nabla\LL_r}_2\norm{\nabla \LL_f}_2}
        \end{align}
        
        We can rewrite $\norm{\nabla\LL_D^{aw}}_2$ in term of $\angle_2 \left(\nabla \LL_r, \nabla \LL_f\right)$, that is
        \begin{align}
            \norm{\nabla\LL_D^{aw}}_2^2&=\la \nabla\LL_D^{aw},\nabla\LL_D^{aw}\ra_2 \nonumber\\
            &=\la \frac{1}{\norm{\nabla \LL_r}_2}\nabla\LL_r+\frac{1}{\norm{\nabla \LL_f}_2}\nabla\LL_f, \frac{1}{\norm{\nabla \LL_r}_2}\nabla\LL_r+\frac{1}{\norm{\nabla \LL_f}_2}\nabla\LL_f\ra_2 \nonumber\\
            &=\frac{\la \nabla \LL_r,\nabla \LL_r\ra_2}{\norm{\nabla \LL_r}_2^2}+\frac{\la \nabla \LL_f,\nabla \LL_f\ra_2}{\norm{\nabla \LL_f}_2^2}+\frac{2\la \nabla \LL_r,\nabla \LL_f\ra_2}{\norm{\nabla \LL_r}_2\norm{\nabla \LL_f}_2}\nonumber\\
            &=2\left(1+\cos\left(\angle_2 \left(\nabla \LL_r, \nabla \LL_f\right)\right) \right)\label{supp_mat:pf1:c1:eq2}.
        \end{align}
        
        Notice that (\ref{supp_mat:pf1:c1:eq1}) can be rewritten using $\angle_2 \left(\nabla \LL_r, \nabla \LL_f\right)$ as 
        \begin{align}
            \cos\left(\angle_2 \left(\nabla \LL_D^{aw}, \nabla \LL_r\right)\right)&=\frac{1}{\norm{\nabla\LL_D^{aw}}_2}+\frac{\la \nabla\LL_r, \nabla\LL_f\ra_2}{\norm{\nabla\LL_D^{aw}}_2\norm{\nabla\LL_r}_2\norm{\nabla \LL_f}_2}\\
            &=\frac{1}{\norm{\nabla\LL_D^{aw}}_2}\left(1+\cos\left(\angle_2 \left(\nabla \LL_r, \nabla \LL_f\right)\right) \right)\\
            &=\sqrt{\frac{1+\cos\left(\angle_2 \left(\nabla \LL_r, \nabla \LL_f\right)\right) }{2}}\\
            &=\cos\left(\angle_2 \left(\nabla \LL_r, \nabla \LL_f\right)/2 \right).
        \end{align}
        
        Thus, $\angle_2 \left(\nabla \LL_D^{aw}, \nabla \LL_r\right)=\angle_2 \left(\nabla \LL_D^{aw}, \nabla \LL_f\right)=\angle_2 \left(\nabla \LL_r, \nabla \LL_f\right)/2$.
        \item If $w_r=\frac{1}{\norm{\nabla \LL_r}_2}$ and $w_f=-\frac{\la\nabla\LL_r,\nabla\LL_f\ra_2}{\norm{\nabla \LL_f}_2^2\norm{\nabla \LL_r}_2}$ then
        \begin{equation}
            \LL_D^{aw}=\frac{1}{\norm{\nabla \LL_r}_2}\LL_r-\frac{\la\nabla\LL_r,\nabla\LL_f\ra_2}{\norm{\nabla \LL_f}_2^2\norm{\nabla \LL_r}_2}\LL_f.
        \end{equation}
        Using this aw-loss function, we have
        \begin{align}
            \la \nabla \LL_D^{aw}, \nabla \LL_f\ra_2 &= \la \frac{1}{\norm{\nabla \LL_r}_2}\nabla \LL_r-\frac{\la\nabla\LL_r,\nabla\LL_f\ra_2}{\norm{\nabla \LL_f}_2^2\norm{\nabla \LL_r}_2}\nabla\LL_f,\nabla\LL_f\ra_2\\
            &=\frac{\la\nabla\LL_r,\nabla\LL_f\ra_2}{\norm{\nabla \LL_r}_2}-\frac{\la\nabla\LL_r,\nabla\LL_f\ra_2\la\nabla\LL_f,\nabla\LL_f\ra_2}{\norm{\nabla \LL_f}_2^2\norm{\nabla \LL_r}_2}\\
            &=\frac{\la\nabla\LL_r,\nabla\LL_f\ra_2}{\norm{\nabla \LL_r}_2}-\frac{\la\nabla\LL_r,\nabla\LL_f\ra_2}{\norm{\nabla \LL_r}_2}=0,
        \end{align}
        and
        \begin{align}
            \la \nabla \LL_D^{aw}, \nabla \LL_r\ra_2 &= \la \frac{1}{\norm{\nabla \LL_r}_2}\nabla \LL_r-\frac{\la\nabla\LL_r,\nabla\LL_f\ra_2}{\norm{\nabla \LL_f}_2^2\norm{\nabla \LL_r}_2}\nabla\LL_f,\nabla\LL_r\ra_2\\
            &= \frac{\norm{\nabla \LL_r}_2^2}{\norm{\nabla \LL_r}_2}-\frac{\la\nabla\LL_r,\nabla\LL_f\ra_2^2}{\norm{\nabla \LL_f}_2^2\norm{\nabla \LL_r}_2}\\
            &\geq\norm{\nabla \LL_r}_2 - \frac{\norm{\nabla \LL_f}_2^2\norm{\nabla \LL_r}_2^2}{\norm{\nabla \LL_f}_2^2\norm{\nabla \LL_r}_2}\\
            &=\norm{\nabla \LL_r}_2-\norm{\nabla \LL_r}_2=0.
        \end{align}

        Thus, $\angle_2\left(\nabla \LL_D^{aw}, \nabla \LL_f\right)=90^{\circ}$ and $\angle_2\left( \nabla \LL_D^{aw}, \nabla \LL_r\right) \leq 90^{\circ}$.
        \item If $w_r=-\frac{\la\nabla\LL_r,\nabla\LL_f\ra_2}{\norm{\nabla \LL_r}_2^2\cdot\norm{\nabla \LL_f}_2}$ and $w_f=\frac{1}{\norm{\nabla\LL_f}_2}$ then
        \begin{equation}
            \LL_D^{aw}=-\frac{\la\nabla\LL_r,\nabla\LL_f\ra_2}{\norm{\nabla \LL_r}_2^2\cdot\norm{\nabla \LL_f}_2}\LL_r+\frac{1}{\norm{\nabla\LL_f}_2}\LL_f,
        \end{equation}
        and similar argument as above proves that $\angle_2\left(\nabla \LL_D^{aw}, \nabla \LL_r\right)=90^{\circ}$ and $\angle_2\left( \nabla \LL_D^{aw}, \nabla \LL_f\right)\leq 90^{\circ}$.
    \end{enumerate}
\end{proof}

\clearpage
\subsection{}\label{Supp_Mater:B}
\subsubsection{Ablation study of $\alpha_1$ and $\alpha_2$ parameters}
\begin{figure}[H]
    \centering
    \includegraphics[width=0.45\textwidth]{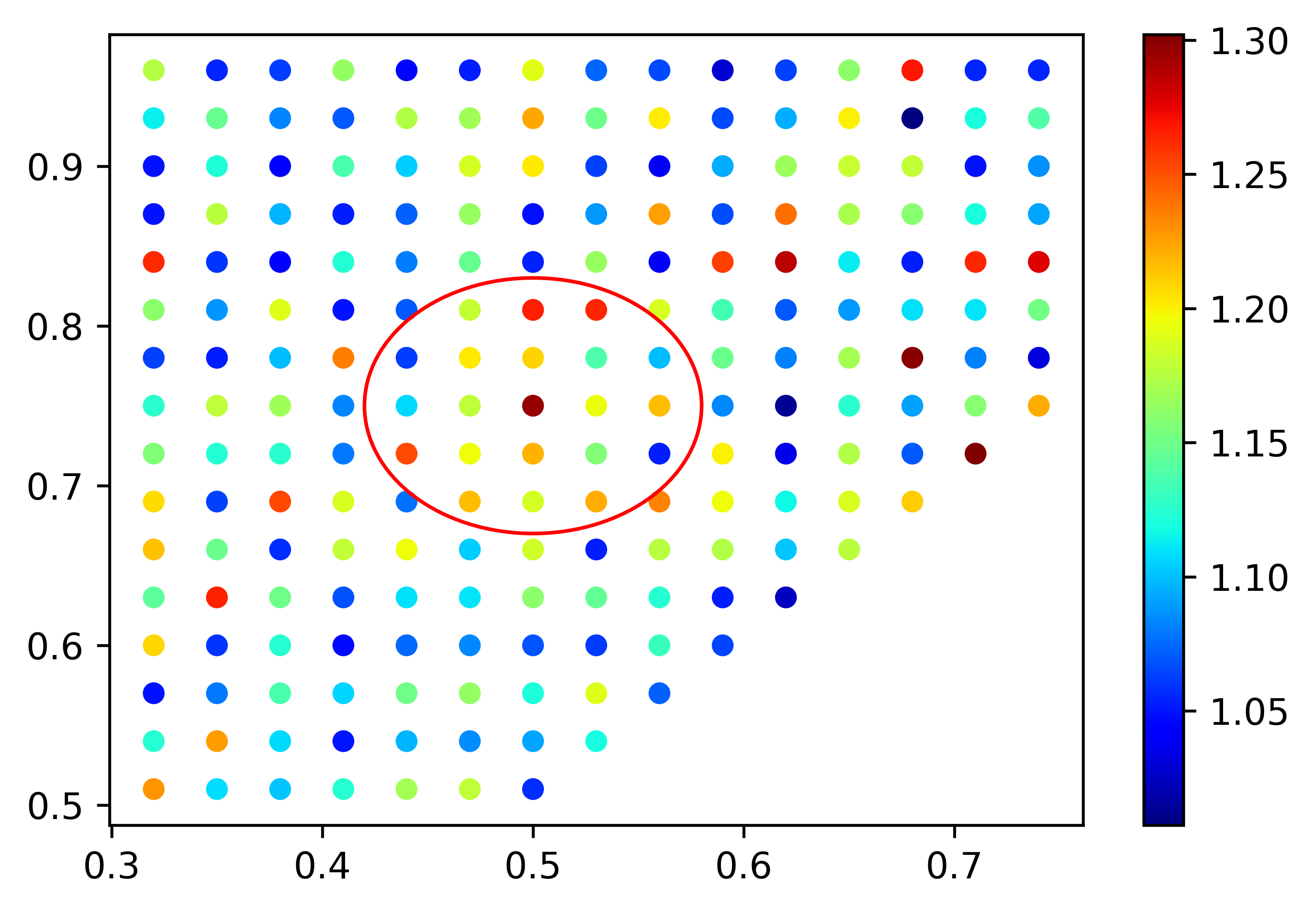}
    \includegraphics[width=0.45\textwidth]{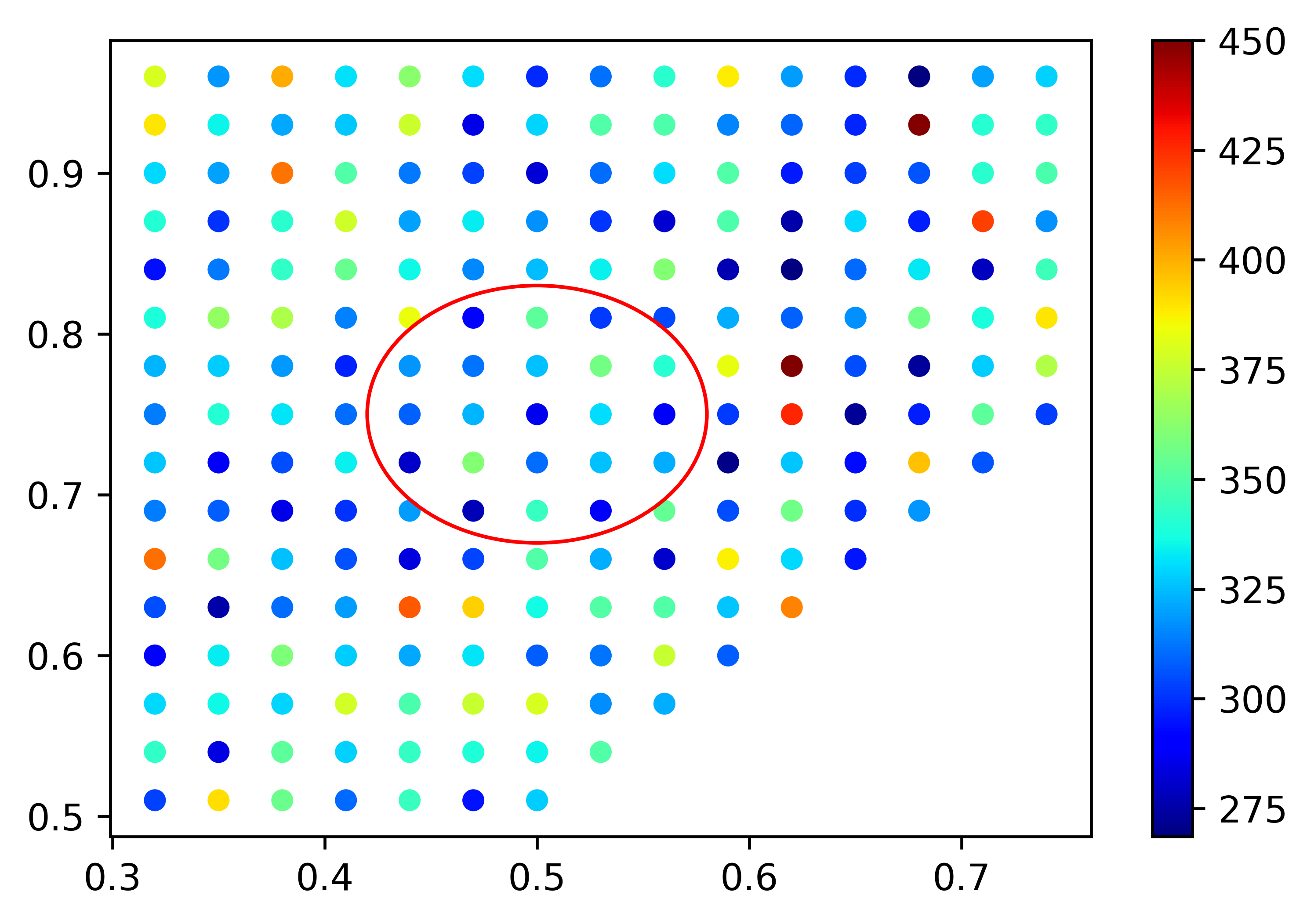}
    \includegraphics[width=0.45\textwidth]{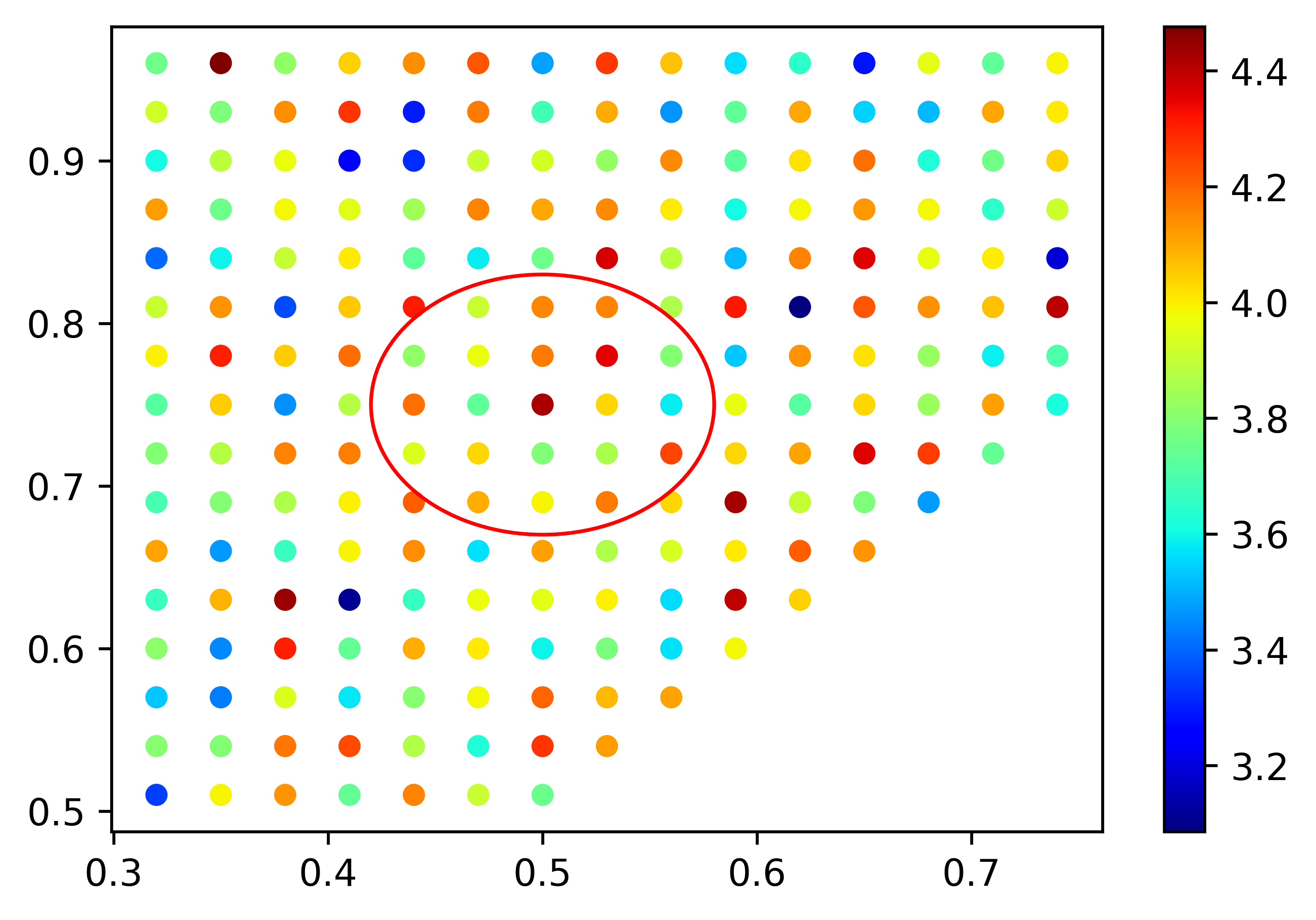}
    \includegraphics[width=0.45\textwidth]{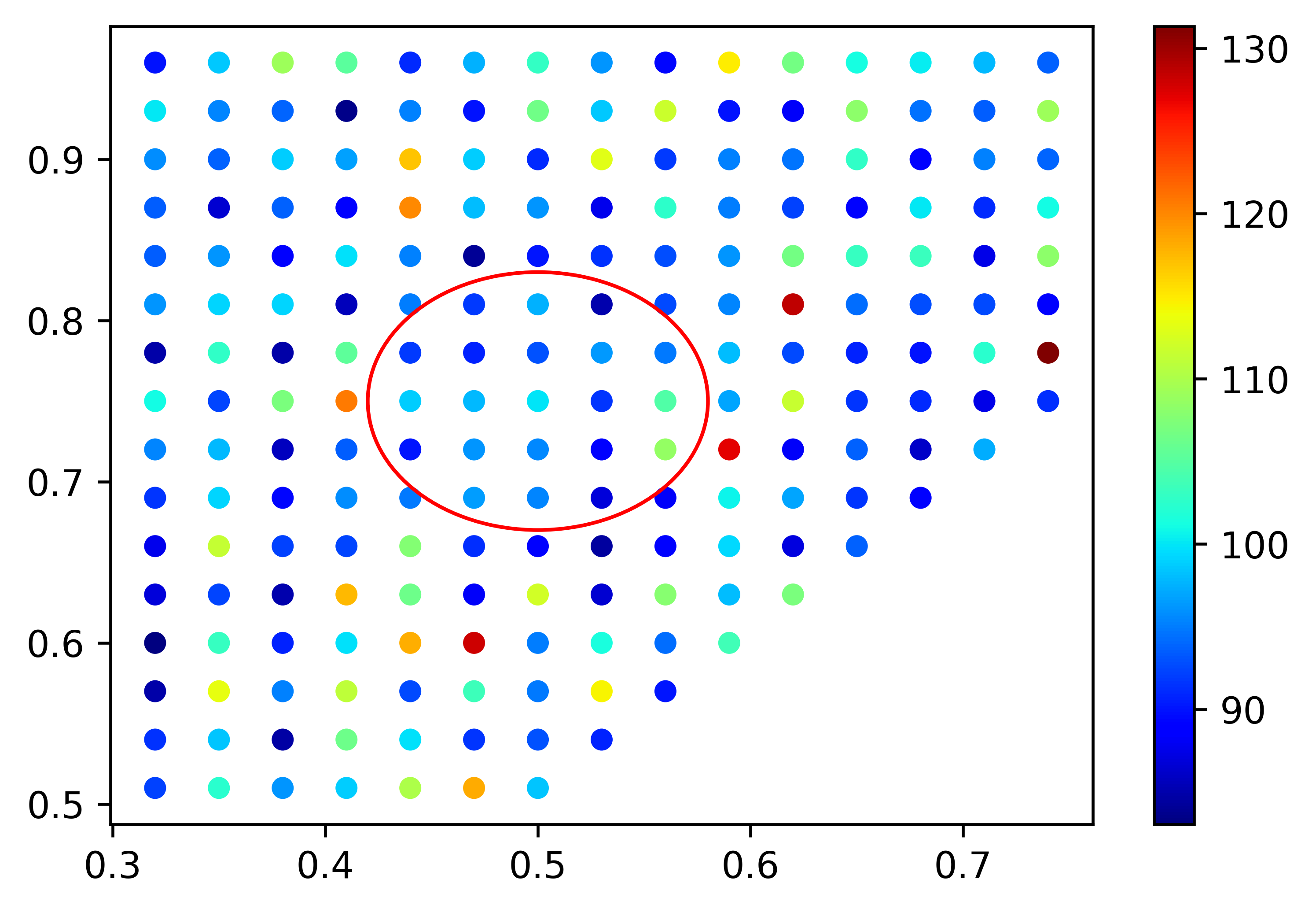}
    \caption{Top-left: IS grid after the first epoch; Bottom-left: IS grid after the fifth epoch; Top-right: FID grid after the first epoch; Bottom-right: FID grid after the fifth epoch.}
    \label{fig:grid}
\end{figure}

We have considered the choice of $\alpha_1$ and $\alpha_2$ by experimenting with aw-AutoGAN models on the CIFAR-10 dataset. We recorded IS and FID scores after the first and the fifth epochs where the models were trained with the parameters in the grids shown in Figure \ref{fig:grid} with all other settings fixed. We present the IS and FID scores in heat map plots in Figure \ref{fig:grid}. The results  show that neighbors around the point $(0.5, 0.75)$ lead to good scores; the point $(0.5, 0.75)$ itself produces one of the highest IS and one of the lowest FID. In particular, the performance is not too sensitive to the selections.

\subsubsection{Dataset and Implementation Details}
We test our aw-method on the following datasets:
\begin{itemize}[leftmargin=*]
    \item The \emph{CIFAR-10} dataset~\cite{Krizhevsky09learningmultiple}  consists of 60,000 color images with 50,000 for training and 10,000 for testing. All images have resolution $32\times32$ pixels and are divided equally into 10 classes, with 6,000 images per class. No data augmentation;
    \item The \emph{STL-10} is a dataset proposed in~\cite{pmlr-v15-coates11a} and designed for image recognition and unsupervised learning. STL-10 consists of 100,000 unlabeled images with $96 \times 96$ pixels and split into 10 classes. All images are resized to $48 \times 48$ pixels, without any other data augmentation;
    \item The \emph{CIFAR-100} from~\cite{Krizhevsky09learningmultiple} is a dataset similar to CIFAR-10 that consists of 60,000 color $32 \times 32$ pixel images that are divided into 100 classes. No data augmentation.
\end{itemize}

We follow the original implementations of SN-GAN, AutoGAN and BigGAN-PyTorch~\cite{BG_github} that use the following hyperparameters:
\begin{itemize}[leftmargin=*]
    \item \emph{Generator: }learning rate: 0.0002; batch size: 128 (SN-GAN, AutoGAN) and 50 (BigGAN); optimizer: Adam optimizer with $\beta_1=0$ and $\beta_2=0.999$~\cite{Adam_opt};  loss: hinge~\cite{lim2017geometric, tran2017hierarchical}; spectral normalization: False; learning rate decay: linear; \# of training epochs: 320 (SN-GAN), 300 (AutoGAN) and 1000 (BigGAN);
    \item \emph{Discriminator: }learning rate: 0.0002; batch size: 64 (SN-GAN, AutoGAN) and 50 (BigGAN); optimizer: Adam optimizer with $\beta_1=0$ and $\beta_2=0.999$; loss: hinge; spectral normalization: True; learning rate decay: linear; training iterations ratio: 3 (SN-GAN) and 2 (AutoGAN, BigGAN).
\end{itemize}

Experiments based on SN-GAN and AutoGAN models are performed on a single NVIDIA$^{\circledR}$ QUADRO$^{\circledR}$ P5000 GPU running Python 3.6.9 with PyTorch v1.1.0 for AutoGAN based models and Chainner v4.5.0 for SN-GAN based models. Experiments based on BigGAN-Pytorch~\cite{BG_github} model are performed on two NVIDIA$^{\circledR}$ Tesla$^{\circledR}$ V100 GPU running Python 3.6.12 with PyTorch v1.4.0.

\subsubsection{Aw-method with non-normalized gradients}

In section \ref{sect:awGAN_section}, we introduced Algorithm \ref{alg:1} which was developed using Theorem \ref{thm:normalized} where the weights $w_r$ and $w_f$ include normalization of the gradients $\nabla \LL_r$ and $\nabla\LL_f$. This is not necessary and we can consider using  non-normalized gradients directly where one of the weights is chosen as 1. The corresponding results are stated as the following theorem. 

\begin{thm}\label{thm:not_normalized}
    Consider $\LL_D^{aw}$ in (\ref{intro:eq:02}) and the gradient $\nabla \LL_D^{aw}$.
    \begin{enumerate}[leftmargin=*]
        \item If $w_r=1$ and $w_f=-\frac{\la\nabla\LL_r,\nabla\LL_f\ra_2}{\norm{\nabla \LL_f}_2^2}$, then
        \begin{equation}
            \angle_2\left(\nabla \LL_D^{aw}, \nabla \LL_f\right)=90^{\circ},\; \angle_2\left( \nabla \LL_D^{aw}, \nabla \LL_r\right) \leq 90^{\circ}.
        \end{equation}
        \item If $w_r=-\frac{\la\nabla\LL_r,\nabla\LL_f\ra_2}{\norm{\nabla \LL_r}_2^2}$ and $w_f=1$, then
        \begin{equation}
            \angle_2\left( \LL_D^{aw}, \nabla \LL_r\right)=90^{\circ},\; \angle_2\left( \nabla \LL_D^{aw}, \nabla \LL_f\right) \leq 90^{\circ}.
        \end{equation}
    \end{enumerate}
\end{thm}

\begin{proof}
    Identical to the proof of Theorem \ref{thm:normalized}.
\end{proof}

Similar to section \ref{sect:awGAN_section}, we have developed Algorithm \ref{alg:2} using Theorem \ref{thm:not_normalized}. The key difference between Algorithms \ref{alg:1} and \ref{alg:2} is normalization of the gradients; the rest of the algorithm is unchanged including the values for $\alpha_1=0.5$, $\alpha_2=0.75$, $\varepsilon=0.05$ and $\delta=0.05$. 

\IncMargin{0.15em}\SetNlSty{text}{}{:}
\begin{algorithm}
    \SetAlgoLined
        \textbf{Given:} $\PP_\data$ and $\PP_\model$ - data and model distributions\;
        \textbf{Given:} $\alpha_1=0.5$, $\alpha_2=0.75$, $\varepsilon=0.05$, $\delta=0.05$\;
        \textbf{Sample:} $x_1, \ldots, x_n \sim \PP_\data$ and $y_1, \ldots, y_n \sim \PP_\model$\;
        \textbf{Compute:} $\nabla\LL_r$, $\nabla\LL_f$, $s_r=\frac{1}{n}\sum_{i=1}^n\sigma(D(x_i))$, $s_f=\frac{1}{n}\sum_{j=1}^n\sigma(D(y_j))$\;
        \uIf{$s_r < s_f-\delta$ or $s_r < \alpha_1$}{
            \uIf{$\angle_2\left(\nabla\LL_r,\nabla\LL_f\right) > 90^{\circ}$}{
            $w_r=1+\varepsilon$; $w_f=-\frac{\la\nabla\LL_r,\nabla\LL_f\ra_2}{\norm{\nabla \LL_f}_2^2}+\varepsilon$\;
            }
            \Else{
            $w_r=1+\varepsilon$; $w_f=\varepsilon$\;
            }
        }
        \uElseIf{$s_r > s_f-\delta$ and $s_r > \alpha_2$}{
            \uIf{$\angle_2\left(\nabla\LL_r,\nabla\LL_f\right) > 90^{\circ}$}{
            $w_r=-\frac{\la\nabla\LL_r,\nabla\LL_f\ra_2}{\norm{\nabla \LL_r}_2^2}+\varepsilon$;
            $w_f=1+\varepsilon$\;
            }
            \Else{
            $w_r=\varepsilon$; $w_f=1+\varepsilon$\;
            }
        }
        \Else{
        $w_r=1+\varepsilon$;  $w_f=1+\varepsilon$\;
        }
    \caption{Adaptive weighted discriminator method w/o normalization for one step of discriminator training.}
    \label{alg:2}
\end{algorithm}

Similar to Section \ref{sect:experiments}, we tested Algorithm \ref{alg:2} on unconditional and conditional image generating tasks, with results provided in Tables \ref{table:not_normalized} and \ref{table:conditional_gan_ext}, respectively. The implementation details are the same as in Section \ref{sect:experiments}.  

For the unconditional GAN, we tested our non-normalized aw-method on SN-GAN and AutoGAN baselines on three datasets: CIFAR-10, STL-10, and CIFAR-100. Our non-normalized methods also significantly improve SN-GAN baseline, in particular achieving the highest IS for the STL-10 dataset among comparisons in Table \ref{stl10_results}, and aw-AutoGAN non-normalized models improve the baseline AutoGAN models as well on all the datasets.

For the conditional GAN, we tested our non-normalized aw-method on SN-GAN and BigGAN models as baselines on CIFAR-10 and CIFAR-100 datasets. Both non-normalized aw-SN-GAN and aw-BigGAN significantly improve baseline models for both datasets.

On average, normalized weights achieve better results than non-normalized ones. We advocate the normalized version (Algorithm \ref{alg:1}), but both produce quite competitive results and should be considered in implementations. 

\begin{table}[H]
    \begin{center}
        \begin{tabular}{l||c|c||c|c||c|c}
            & \multicolumn{2}{c||}{CIFAR-10} & \multicolumn{2}{c||}{STL-10} & \multicolumn{2}{c}{CIFAR-100}\\
            \hline
            \rowcolor{lightgray}Method & IS $\uparrow$ & FID $\downarrow$ & IS $\uparrow$ & FID $\downarrow$ & IS $\uparrow$ & FID $\downarrow$ \\
            \hline
            SN-GAN~\cite{miyato2018spectral} & 8.22{\scriptsize$\pm$.05} & 21.7 & 9.10{\scriptsize$\pm$.04} & 40.10 & 8.18{\scriptsize$\pm$.12}$^{*}$ & 22.40$^{*}$ \\
            \hline
            aw-SN-GAN (Ours) &  8.53{\scriptsize$\pm$.11} & 12.32 & 9.53{\scriptsize$\pm$.10} & 36.41 & 8.31{\scriptsize$\pm$.02} & 19.08 \\
            \hline
            aw-SN-GAN (non-norm.; Ours) & 8.43{\scriptsize$\pm$.07} & 12.65 & 9.61{\scriptsize$\pm$.12} & 34.72 & 8.30{\scriptsize$\pm$.11} & 19.48 \\
            \hline\hline
            AutoGAN~\cite{Gong_2019_ICCV} & 8.55{\scriptsize$\pm$.10} & 12.42 & 9.16{\scriptsize$\pm$.12} & 31.01 & 8.54{\scriptsize$\pm$.10}$^{*}$ & 19.98$^{*}$ \\
            \hline
            aw-AutoGAN (Ours) &  9.01{\scriptsize$\pm$.03} & 11.82  &  9.41{\scriptsize$\pm$.09} & 26.32 & 8.90{\scriptsize$\pm$.06} & 19.00 \\ 
            \hline
            aw-AutoGAN (non-norm.; Ours) & 8.98{\scriptsize$\pm$.06} & 13.17 & 9.59{\scriptsize$\pm$.14} & 27.97 & 8.72{\scriptsize$\pm$.05} & 19.89 
        \end{tabular}
    \end{center}
    \caption{Unconditional GAN: CIFAR-10, STL-10, and CIFAR-100 scores for the normalized (Algorithm \ref{alg:1}) and non-normalized (Algorithm \ref{alg:2}) versions of aw-method; $^{*}$ - results from our test.}
    \label{table:not_normalized}
\end{table}

\begin{table}[H]
    \begin{center}
        \begin{tabular}{l||c|c||c|c}
            & \multicolumn{2}{c||}{CIFAR-10} & \multicolumn{2}{c}{CIFAR-100}\\
            \hline
            \rowcolor{lightgray}Method & IS $\uparrow$ & FID $\downarrow$ & IS $\uparrow$ & FID $\downarrow$\\
            \hline
            SN-GAN (cond.)~\cite{miyato2018spectral} & 8.60{\scriptsize$\pm$.08} & 17.5 & 9.30$^{\dag}$ & 15.6$^{\dag}$ \\
            \hline
            aw-SN-GAN (cond.; Ours) & 9.03{\scriptsize$\pm$.11} & 8.11 & 9.48{\scriptsize$\pm$.13} & 14.42 \\
            \hline
            aw-SN-GAN (non-norm.; cond.; Ours) & 9.00{\scriptsize$\pm$.12} & 8.03 & 9.44{\scriptsize$\pm$.16} & 14.00\\
            \hline \hline
            BigGAN~\cite{brock2018large} & 9.22 & 14.73 & 10.99{\scriptsize$\pm$.14}$^{*}$ & 11.73$^{*}$ \\
            \hline
            aw-BigGAN (cond.; Ours) & 9.52{\scriptsize$\pm$.10} & 7.03 & 11.22{\scriptsize$\pm$.17} & 10.23 \\
            \hline
            aw-BigGAN (non-norm.; cond.; Ours) & 9.50{\scriptsize$\pm$.07} & 6.89 & 11.26{\scriptsize$\pm$.20} & 10.25 \\
        \end{tabular}
    \end{center}
    \caption{Conditional GAN: CIFAR-10 and CIFAR-100 scores for the normalized (Algorithm \ref{alg:1}) and non-normalized (Algorithm \ref{alg:2}) versions of aw-method; $^{\dag}$ - quoted from~\cite{shmelkov2018good}; $^{*}$ - results from our tests based on~\cite{BG_github}.}
    \label{table:conditional_gan_ext}
\end{table}
\end{document}